\relax
\documentclass[conference]{IEEEtran} 
\usepackage{helvet}  
\usepackage{courier}  
\usepackage{cite}
\author{Xuchao Zhang{$^\dag$}, Shuo Lei{$^\dag$}, Liang Zhao{$^\ddag$}, Arnold P. Boedihardjo{$^\S$}, Chang-Tien Lu{$^\dag$} \\
	{$^\dag$}Virginia Tech, Falls Church, VA, USA\\
	{$^\ddag$}George Mason University, Fairfax, VA, USA\\
	{$^\S$}U. S. Army Corps of Engineers, Alexandria, VA, USA \\
	{$^\dag$}\{xuczhang, slei, ctlu\}@vt.edu, {$^\ddag$}lzhao9@gmu.edu, {$^\S$}arnold.p.boedihardjo@usace.army.mil}

\usepackage{graphicx, subfigure}
\graphicspath{ {./images/} }

\usepackage{url}

\usepackage{mathtools}
\usepackage{amsmath}
\usepackage{amssymb}
\usepackage{bbm}
\usepackage{bm}
\DeclareMathOperator*{\argminA}{arg\,min}
\DeclareMathOperator*{\argmaxA}{arg\,max}

\DeclareMathOperator*{\supp}{supp}

\newcommand{\argmin}{\arg\min}
\DeclarePairedDelimiter\ceil{\lceil}{\rceil}

\DeclarePairedDelimiter\abs{\lvert}{\rvert}%
\DeclarePairedDelimiter\norm{\lVert}{\rVert}%

\newcommand\nth{\textsuperscript{th}\xspace}

\usepackage{amsthm}
\newtheorem{theorem}{Theorem}
\newtheorem{definition}{Definition}
\newtheorem{lemma}[theorem]{Lemma}

\usepackage[small]{caption}
\usepackage{color}

\usepackage{tabularx, booktabs}
\usepackage{diagbox}
\newcolumntype{Y}{>{\centering\arraybackslash}X}

\usepackage{times}

\usepackage[linesnumbered,ruled,vlined,algo2e]{algorithm2e}

\begin{document}
%
\title{Robust Regression in Online Feature Selection}

\title{RoOFS: Robust Regression via Online Feature Selection}
\title{Robust Regression via Online Feature Selection under Adversarial Data Corruption}



\maketitle

\begin{abstract}

The presence of data corruption in user-generated streaming data, such as social media, motivates a new fundamental problem that learns reliable regression coefficient when features are not accessible entirely at one time. Until now, several important challenges still cannot be handled concurrently: 1) corrupted data estimation when only partial features are accessible; 2) online feature selection when data contains adversarial corruption; and 3) scaling to a massive dataset. This paper proposes a novel RObust regression algorithm via Online Feature Selection (\textit{RoOFS}) that concurrently addresses all the above challenges. Specifically, the algorithm iteratively updates the regression coefficients and the uncorrupted set via a robust online feature substitution method. We also prove that our algorithm has a restricted error bound compared to the optimal solution. Extensive empirical experiments in both synthetic and real-world data sets demonstrated that the effectiveness of our new method is superior to that of existing methods in the recovery of both feature selection and regression coefficients, with very competitive efficiency.
	
\end{abstract}

\section{Introduction} \label{section:introduction}
The presence of noise and data corruption in real-world data can be inevitably caused by various reasons such as experimental errors, accidental outliers, or even adversarial data attacks. 
In traditional robust regression problem, reliable regression coefficients are learned in the presence of adversarial data corruptions in its response vector. A commonly adopted model from existing methods assumes that the observed response is obtained from the generative model $\bm y = X^T \bm \beta^* + \bm u$, where $\bm \beta^*$ is the true regression coefficients we wish to recover and $\bm u$ is the corruption vector with adversarial values. In the problem setting, the data matrix $X$ is assumed to contain all the features that can be accessed at any time and by arbitrarily many times. 

Existing robust learning methods typically focus on modeling the entire dataset with all the features at once; however, they may meet the bottleneck in terms of computation and memory as more and more data sets are becoming 
However, the assumption is no longer suitable to the following scenarios in the applications that contain exponentially increasing user-generated contents:
1) \textit{features are too many to be loaded entirely}. Features grows dramatically fast and becomes extremely large in. For instance, over 4.7 million movies and televisions with 8.3 million reviews in IMDb\footnote{https://www.imdb.com/} online movie and television review website, which makes it hard to load all the features entirely for any machine learning models using the movies as features. 
2) \textit{features are generated dynamically}. For example, people create and use new terms and hashtags all the time in Twitter, and the "Likes"~\cite{naylor2012beyond} on newly-generated articles in Facebook can be considered as new features describing the interestingness of the user.
3) 
Therefore, it is necessary to address online features in traditional robust regression as a new fundamental problem; however, current methods either focus on robust regression or online feature learning separately.



To the best of our knowledge, our proposed approach is the first robust regression algorithm that can handle the online features with adversarial data corruptions.
It is nontrivial to consider online features and adversarial data corruption simultaneously in robust regression because 1) robust methods usually estimate data corruption based on the entire data, but online features make the data can only be partially accessible at one time; and 2) online feature selection methods can only select features based on uncorrupted data. Simply using robust regression and online feature selection methods sequentially makes the recovery result of coefficients worse, which is presented in our experiments in Section \ref{section:experiment}.
To address the above challenges, we proposed a new robust regression algorithm via online feature selection (\textit{RoOFS}). The main contributions of our study are summarized as follows:
\begin{itemize}
	\item \textit{design of an efficient algorithm to simultaneously address the problem of data corruption and online feature}. The algorithm \textit{RoOFS} is proposed to recover the regression coefficients and uncorrupted set efficiently. Unlike using entire features, our approach alternately estimates the data corruption and selects the feature set via a robust online feature substitution method.
	\item \textit{theoretical analysis of the algorithm}. We prove that our method yields a solution with a restricted error bound compared to ground truth coefficients under the Subset Restricted Strong Convexity (\textit{SRSC}) property.
	\item \textit{demonstration of empirical effectiveness and efficiency}. Our proposed algorithm was evaluated with 6 competing methods in both robust regression and online feature selection literatures. The results showed that our approach consistently outperforms existing methods in coefficients recovery and uncorrupted set estimation, delivering a competitive running time.
\end{itemize}

The reminder of this paper is organized as follows. Section \ref{section:related_work} reviews the related work in robust regression model and online feature selection categories. Section \ref{section:problem_setting} gives a formal problem formulation. The proposed \textit{RoOFS} algorithm is presented in Section \ref{section:methodology}. Section \ref{section:analysis} presents the theoretical analysis of proposed algorithm. In Section \ref{section:experiment}, the experimental results are analyzed and the paper concludes with a summary of our work in Section \ref{section:conclusion}.

\section{Related Work} \label{section:related_work}
The work related to this paper is summarized in the categories of robust regression model and online feature selection as below.
\subsection{Robust Regression Model}
A large body of literature on robust regression problem has been established over the last few decades. Most of studies focus on handling stochastic noise in small amounts \cite{loh2011high}; however, these methods cannot be applied to data that may exhibit malicious corruption \cite{icml2013_chen13h}. To recover regression coefficients with adversarial data corruption, Chen et al. \cite{icml2013_chen13h} proposed a robust algorithm based on trimmed inner product. McWilliams et al. \cite{mcwilliams2014fast} proposed a sub-sampling algorithm for large-scale corrupted linear regression, but their theoretical recovery boundaries are not close to the ground truth \cite{bhatia2015robust}. Some $L_1$ penalty based methods \cite{Wright:2010:DEC:1840493.1840533,nguyen2013exact} pursue strong recovery results for robust regression problem, but these methods depend on severe restrictions of the data distribution such as row-sampling from an incoherent orthogonal matrix \cite{nguyen2013exact}. Zhang et al. \cite{8215535} proposed a distributed robust algorithm to handle the large-scale data set under adversarial data corruption.

Most research in this area requires the corruption ratio parameter, which is difficult to estimate under the assumption that the dataset can be adversarially attacked. For instance, She and Owen \cite{10.2307/41416397} rely on a regularization parameter to determine the size of the uncorrupted set based on soft-thresholding. Chen et al. \cite{icml2013_chen13h} require the upper bound of the outliers number, which is also difficult to estimate when the data contain the adversarial data corruption. Bhatia et al. \cite{bhatia2015robust} proposed a hard-thresholding algorithm with a strong guarantee of coefficient recovery under mild assumption on input data. However, the corruption ratio parameter is required by the algorithm and its recovery error can be more than doubled in size if the parameter is far from the true value. Recently, Zhang et al. \cite{ijcai2017-480} proposed a heuristic hard-thresholding based methods that learns the optimal uncorrupted set. However, all these approaches are based on batch feature selection under the assumption that all features can be accessed entirely at any time, which is infeasible to apply in massive and fast growing feature set.

\subsection{Online Feature Selection}
Online feature selection methods \cite{jiang2006similarity,wang2014online,Yu:2016:SAO:3017677.2976744} relaxes the requirement of batch selection and fit the scenarios that feature cannot be accessed entirely at one time. Statistical online feature selection algorithms \cite{zhou2006streamwise,wu2010online,yu2014towards} select features via certain statistical quantity such as mutual information, but these methods lack of specific objectives and usually have sub-optimal solutions for some certain tasks. Optimization based approaches \cite{perkins2003online,zhu2010grafting} use target oriented objective functions solved by some specific optimization techniques. These methods usually require the regression coefficient $\bm \beta$ be sparse, i.e., $\norm{\bm \beta}_0 \le \mu$. Grafting \cite{Perkins:2003:OFS:3041838.3041913} and its variation \cite{zhu2010grafting} relax the hard constraint of feature set into $L_1$ penalty, which makes it a convex problem. However, the parameter of $L_1$ norm \cite{ryali2009feature} is difficult to determine because the usual cross validation strategy is unavailable for the online feature selection scenario \cite{wang2015online}. Yang et al. \cite{Yang:2016:OFS:2939672.2939881} proposed a limited-memory substitution algorithm based on the $L_0$ norm constraint. Although the hard constraint leads to an NP-hard problem, a theoretical guarantee for the error bound of their local optimal solution is provided. However, none of these online feature methods can handle the adversarial data corruption.

\section{Problem Formulation}\label{section:problem_setting}
In this study, we consider the problem of robust regression with adversarial data corruption in the feature selection scenario in which only a few features are accessible at each time. Given data matrix $X_t\in \mathbbm{R}^{{p_t}\times n}$ where $p_t$ is the number of features available in the $t\nth$ time interval, and $n$ are the number of data samples. The data matrix for all the time intervals is represented as $X = \{X_t\}_{t=1}^{\mathcal T}$. We assume the corresponding response vector $\bm y \in \mathbbm{R}^{n \times 1}$ is generated using the following model:
\begin{equation} \label{eq:model}
\bm y = X^T \boldsymbol \beta^* + \bm u + \bm \varepsilon
\end{equation}
where $\bm \beta^*$ represents the $\mu$-sparse ground truth coefficients of the regression model i.e., $\norm{\bm \beta^*}_0 \le \mu$ and $\bm u$ is the unbounded corruption vector introduced by adversarial data attacks. $\bm \varepsilon \in \mathbbm{R}^{n\times 1}$ represents the additive dense noise, where $\varepsilon_i \sim \mathcal{N}(0, \sigma^2)$. Different from the corruption vector $\bm u$ that can be arbitrarily distributed, the dense noise $\varepsilon_i$ follows normal distribution with zero mean and a relatively small variance $\sigma$. The notations used in this paper is summarized in Table \ref{table:math_notation}.

\begin{table}[t]
	\caption{Math Notations}
	\centering
	\label{table:math_notation}
	\tabcolsep=0.15cm
	\scalebox{0.97}{
	\begin{tabular}{ l|l }
		\toprule
		Notations                                          & Explanations                                                                 \\ \hline
		$p, n \in \mathbbm{R}$ & number of entire features and data samples \\
		$p_t \in \mathbbm{R}$ & number of features in $t\nth$ time interval \\
		$\mu \in \mathbbm{R}$ & ratio of feature sparsity, where $\norm{\bm \beta}_0 = \mu$ \\
		$X_t \in \mathbbm{R}^{p_t\times n}$              & data samples containing features in the $t\nth$ time interval                        \\
		$X \in \mathbbm{R}^{p\times n}$              & data samples containing the entire features                        \\
		$\bm \beta, \bm \beta^* \in \mathbbm{R}^{p\times1}$       & estimated and ground truth regression coefficient                        \\
		$\bm u \in \mathbbm{R}^{n\times1}$           & corruption vector with adversarial values                                         \\
		$\bm \varepsilon \in \mathbbm{R}^{n\times1}$ & dense noise vector, where $\varepsilon_i \sim \mathcal{N}(0, \sigma^2)$                                      \\
		$\bm y \in \mathbbm{R}^{n\times1}$           & response vector, where $\bm y = X^T \boldsymbol \beta^* + \bm u + \bm \varepsilon$                                   \\
		$\bm r \in \mathbbm{R}^{n\times1}$           & residual vector, where $\bm r = \abs{\bm y - X^T\bm \beta}$                                              \\		
		$ S \subseteq [n]$                           & estimated uncorrupted set \\
		$ S_* \subseteq [n]$                         & ground truth uncorrupted set, where $S_*=\overline{supp(\bm u)}$ \\
		$\varPsi, \varPsi_* \subseteq [\mu]$ & estimated and ground truth feature set\\
		
		\bottomrule
	\end{tabular}
	}
\end{table}

The goal of our problem is to learn a new robust regression problem with online feature selection, which is to recover the regression coefficients $\bm \beta^*$ and simultaneously determine the uncorrupted point set $\hat{S}$ with sequentially accessible features. The problem is formally defined as follows:

\begin{equation} \label{eq:problem}
\begin{gathered}
\hat{\bm \beta}, \hat{S}= \argminA_{\bm \beta, S}\|\bm y_S - X_{S}^T \bm \beta\|_2^2\\
s.t.\ \ S\subset[n],\ \abs{S} \ge \mathcal{G}(\bm \beta),\ \norm{\bm \beta}_0 \le \mu
\end{gathered}
\end{equation}


Given a subset $S \subset [n]$, $\bm y_S$ restricts the row of $\bm y$ to indices in $S$ and $X_S$ signifies that the columns of $X$ are restricted to indices in $S$. Therefore, we have $\bm y_S \in \mathbbm{R}^{|S| \times 1}$ and $X_S \in \mathbbm{R}^{p \times |S|}$. We use the notation $S_*=\overline{supp(\bm u)}$ to denote the ground truth set of uncorrupted points. Also, for any vector $\bm v \in \mathbbm{R}^n$, the notation $\bm v_S$ represents the $|S|$-dimensional vector containing the components in $S$. The notation $\Psi = supp(\bm \beta)$ is used to represent the set of selected features, resulting in $\abs{\Psi} \le \mu$. Similarly, we use $X_{\Psi}$ to signify the rows of $X$ are restricted to indices in $\Psi$ and $X_{\Psi, S}$ to restrict both the rows and columns in set $\Psi$ and $S$.
The function $\mathcal{G}(\cdot)$ determines the size of uncorrupted data according to the regression coefficients $\bm \beta$, which is explained in Section \ref{section:methodology}. It is worth mentioning that the features of data matrix $X$ in Equation \eqref{eq:problem} cannot be loaded entirely, but they can be accessed partially for each time interval. Therefore, the joint optimization of $\bm \beta$ and $S$ in our problem are very challenging because it amounts to a non-convex discrete optimization problem under the assumption that data matrix $X$ cannot be access entirely at one time.

\section{The Proposed Methodology}\label{section:methodology}

To solve the problem in Equation \eqref{eq:problem} efficiently with the guarantee on the strong recovery of regression coefficients, we propose a novel robust regression algorithm with online feature selection, \textit{RoOFS}. The algorithm is only allowed to access part of features at each time, which are defined as the newly incoming features in our problem. One naive solution to handle the sequentially incoming features is to retain all the features in the memory and then apply traditional robust feature selection methods. However, the solution has two major drawbacks: 1) the feature set can be too large to be retained in the memory, and 2) the algorithm becomes slower and slower when the feature set increases. Therefore, we proposed a new ``robust online substitution" method to decide the retained feature set based on an adaptively estimated corrupted set.
The procedure of robust online substitution is defined as follows:
\begin{itemize}
	\item Update coefficients of retained features $\Psi$ based on the estimated uncorrupted set $S$ as follows: $\bm \beta_\Psi := $$\bm \beta_\Psi - \eta$ $X_{\Psi,S}^T$$(X_{\Psi,S}^T \bm \beta_{\Psi} - \bm y_{S})$, where $\eta$ is the step length.
	\item Retain the top $\mu$ largest (in magnitude) elements in $\bm \beta$ and set the rest to zero. Then all the non-zero features will be kept in the retained feature set $\Psi$.
	\item Compute the residual vector $\bm r \in \mathbbm{R}^{n \times 1}$ with the updated coefficients $\bm \beta$, then estimate the uncorrupted feature set $S$ via a thresholding operator $\mathcal{H}_{\tau}(\bm r)$, where $\tau$ is the estimated size of uncorrupted set.
\end{itemize}
The procedure will be repeatedly executed until the residual vector $\bm r$ converges. The thresholding operator $\mathcal{H}_{\tau}(\cdot)$ is formally defined as follows:

\begin{definition}[\textbf{Thresholding Operator}] \label{def:HTS}
	Defining $\varphi^{-1}_{\bm v}(i)$ as the position of the $i$\nth element in input vector $\bm v$'s ascending order of magnitude and $\tau$ as the threshold parameter, the thresholding operator of $\bm v$ is defined as
	\begin{equation}\label{eq:hht}
	\begin{aligned}
	\mathcal{H}_{\tau}(\bm v) = \{i \in [n]: \varphi^{-1}_{\bm v}(i) \le \tau \}
	\end{aligned}
	\end{equation}
\end{definition}

To estimate the uncorrupted set $S$, the thresholding operator $\mathcal{H}_{\tau}(\cdot)$ generally requires two inputs: residual vector $\bm r$ and the size of uncorrupted set $\tau$. The residual vector $\bm r$ can be computed with coefficients $\bm \beta$ as follows:
\begin{equation} \label{eq:residual}
\bm r = \abs{\bm y - X^T\bm \beta}
\end{equation}

For the size of uncorrupted set, two general cases are discussed. The first case is that the size can be estimated by users based on their prior knowledge on the data. For instance, if we know the data corruption happens rarely, then we can estimate the uncorrupted size as 95\% of the entire data. However, it is hard to obtain prior knowledge on the data in the real-world. Thus, in the second case, we propose a method to adaptively estimate the uncorrupted size based on the residual vector $\bm r$. 
The method follows an intuition that when the coefficient $\bm \beta$ is close to $\bm \beta^*$, the residuals of uncorrupted samples are smaller than those of corrupted samples in strong possiblity. The intuition can be explained by the generative model in Equation \eqref{eq:model}, where the corrupted samples have the residual $\bm r \approx \bm u + \bm \varepsilon$, but the residual of uncorrupted samples only contains the white noise $\bm \varepsilon$.

The estimation of uncorrupted size can be formalized to solve the following problem:
\begin{equation} \label{eq:tau_opt}
\begin{aligned}
\hat{\tau} :=& \argmaxA_{\ceil*{n/2} < \tau \le n} \tau
\ \ \ \ s.t.\ \ r_{\varphi(\tau)} \le \frac{2\tau r_{\varphi(\tau_o)}}{\tau_o}, \tau \in \mathbbm{Z}^+
\end{aligned}
\end{equation}
where $r_{\varphi(k)}$ represents the $k\nth$ elements of residual vector $\bm r$ in ascending order of magnitude. The variable $\tau_o$ in the constraint is defined as an intermediate variable whose $r_{\varphi(\tau_o)}^2$ has the closest value to $\frac{\norm{\bm r_{\mathcal{H}_{\tau'}(\bm r)}}_2^2}{\tau'}$, where $\tau'=\tau - \ceil*{n/2}$ and $\mathcal{H}_{\tau'}(\bm r)$ represent the position set containing the smallest $\tau'$ elements in residual $\bm r$. The problem in Equation \eqref{eq:tau_opt} can be solved by searching from $n$ to $\ceil*{n/2} + 1$ and return the first value $\hat{\tau}$ which satisfies the constraint. It is important to note that the estimation method in Equation \eqref{eq:tau_opt} requires the coefficients $\bm \beta$ to be close to $\bm \beta^*$. Thus, we optimize the uncorrupted set $S$ along with coefficient $\bm \beta$ until both of them converge.

The details of \textit{RoOFS} algorithm are presented in Algorithm \ref{algo:roofs}. In Line 3, the algorithm receives data matrix $X_{\Psi^k}$ with the incoming feature set $\Psi^k$ at time $k$. The new feature set $\Psi^k$ is combined into the retained feature set $\Psi$ in Line 4. For each incoming feature set, the algorithm iteratively optimizes the regression coefficients $\bm \beta$ and the uncorrupted set $S$ until the value of residual vector $\bm r_{S_t}^t$ is converged in Line 14. Specifically, in Line 6, regression coefficients $\bm \beta$ are updated to a better fit for the current estimated feature set $\Psi$ and uncorrupted set $S_t$.
In Line 8, feature set $\Omega$ that contains features with $\abs{\Psi}-\mu$ smallest weights in $\bm \beta$ is selected. Then features in $\Omega$ are removed from the retained feature set $\Psi$ and the weights in $\bm \beta_\Omega$ are reset to zero in Lines 9 and 10.
The residual vector $\bm r$ is updated in Line 11, while the uncorrupted set $S_{t+1}$ is estimated in Line 12 by the thresholding operator. Finally, both coefficients $\bm \beta$ and uncorrupted set $S$ are returned in Line 17.

\begin{algorithm2e}[t]
	\DontPrintSemicolon 
	\KwIn{Corrupted training data \{$\bm x_i, y_i$\}, $i$ = 1...n, feature ratio $\mu$, tolerance $\epsilon$}
	\KwOut{solution $\hat{\bm \beta}$}
	$\bm \beta^0 \leftarrow \bm 0$, $\Psi$ = $\emptyset$, $S_0$ = [n], $t$ $\leftarrow$ 0, $k$ $\leftarrow$ 0 \\
	\Repeat{No more features;}
	{
		Receive features $X_{\Psi^k}$ from the pool $\bar{\Psi}$ with index set $\Psi^k$\\
		$\Psi = \Psi \cup \Psi^k$ \\
		\Repeat{$\|\bm r_{S_{t+1}}^{t+1}-\bm r_{S_{t}}^{t}\|_2 < \epsilon n$}
		{
			$\bm \beta^{t+1}_\Psi \leftarrow \bm \beta^{t}_\Psi - \eta X_{\Psi,S_t}^T(X_{\Psi,S_t}^T \bm \beta_{\Psi}^t - \bm y_{S_t})$\\
			\If{$\abs{\Psi} > \mu$}{
				$\Omega = \argmin_{\Omega \in \Psi}\norm{\bm \beta_\Omega}_1$ \ \ \textit{s.t.} $\abs{\Omega}=\abs{\Psi} - \mu$\\
				$\bm \beta_\Omega = 0$ \\
				$\Psi = \Psi \setminus \Omega$ \\
			}
			$\bm r = \abs{\bm y - X^T\bm \beta}$ \\
			$S_{t+1} \leftarrow$ $\mathcal{H}_{\tau}(\bm r^{t+1})$, where $\tau$ is the estimated uncorrupted size.\\
			$t \leftarrow t+1$ \\
		}
		$k \leftarrow k+1$
	
	}
	
	\textbf{return} $\bm \beta^{t+1}$, $S_{t+1}$
	\caption{{\sc RoOFS Algorithm}}
	\label{algo:roofs}
\end{algorithm2e}

%
%
%
%

\section{Theoretical Analysis} \label{section:analysis}
In this section, we show that the local optimal solution of our algorithm obtains a restricted error bound compared to ground truth solution.
To prove the theoretical properties of our algorithm, we require that the least squares function satisfies the \textit{Subset Restricted Strong Convexity} (SRSC), which is defined as follows:
\begin{definition}[\textbf{SRSC Property}]
The least squares function $f_S(\bm \beta) = \norm{\bm y_S - X_S^T \bm \beta}_2^2$ satisfies Subset Restricted Strong Convexity (SRSC) Property if the following holds for $\forall \bm \beta_1, \bm \beta_2 \in \Omega_\mu$ and $\forall S \in \mathcal{S}_\gamma$:
\begin{equation} \label{eq:SRSC}
\begin{aligned}
f_S(\bm \beta_1) - f_S(\bm \beta_2) \geq \nabla^Tf_S(\bm \beta_2)(\bm \beta_1 - \bm \beta_2)+ \frac{\varphi_\mu}{2}\norm{\bm \beta_1 - \bm \beta_2}_2^2
\end{aligned}
\end{equation}

\end{definition}

To provide the local optimality property of our solution, the following two lemmas are first proved.

\begin{lemma} \label{lemma:residual}
For a given least squares function $f(\bm \beta) = \norm{\bm y - X^T \bm \beta}_2^2$, let residual vector $\bm r = \bm y - X^T \bm \beta$ and $\delta(k)$ be the k-th position of the ascending order in vector $\bm r$, i.e. $r_{\delta(1)} \le r_{\delta(2)} \le ... \le r_{\delta(n)}$. For any $1\le\tau_1<\tau_2\le n$ and $\forall \bm \beta^t \in \Omega_m$, let $S_1=\{\delta(i)|1\le i \le \tau_1\}$ and $S_2=\{\delta(i)|1\le i \le \tau_2\}$. We then have $f_{S_1}(\bm \beta^t) \le f_{S_2}(\bm \beta^t)$. 
\end{lemma}

\begin{proof}
Let $S_3 = \{\delta(i): \tau_1 + 1 \le i \le \tau_2\}$. Clearly, we have $f_{S_2}(\bm \beta^t) = f_{S_1}(\bm \beta^t) + f_{S_3}(\bm \beta^t)$. Moreover, since each element in $S_3$ is larger than any of the element in $S_1$, we have $f_{S_1}(\bm \beta^t) \le f_{S_2}(\bm \beta^t) + \frac{|S_3|}{|S_1|}f_{S_1}(\bm \beta^t)$ $\le \frac{|S_1|}{|S_1|+|S_3|}f_{S_2}(\bm \beta^t) = \frac{\tau_1}{\tau_2}f_{S_2}(\bm \beta^t)\le f_{S_2}(\bm \beta^t)$.
\end{proof}

\begin{lemma} \label{lemma:tau}
	Let $\tau_* = \gamma n$ be the true number of uncorrupted samples and $\tau_t$ be the estimated uncorrupted threshold at the $t$-th iteration. If $\tau_t \le \tau_*$, then $f_{S_t}(\bm \beta^t) \le f_{S_*}(\bm \beta^t)$. If $\tau_t > \tau_*$, then $f_{S_t}(\bm \beta^t) \le \lambda f_{S_*}(\bm \beta^t)$, where $\lambda = \Big[1+\frac{128(1-\gamma)}{2\gamma-1}\Big]$.
\end{lemma}

\begin{proof}
To simplify the notation, the subscripts $t$ that signify the $t$-th iteration will be omitted and the residual vector $\bm r$ is assumed to be sorted in ascending order of magnitude.

We will discuss the $\tau_t$ value in two different conditions compared to the value of $\tau_*$. In the first condition that $\tau_t \le \tau_*$, let $S_t=\{\delta(i)|1\le i \le \tau_t\}$ and $S_*=\{\delta(i)|1\le i \le \tau_*\}$, we have $f_{S_t}(\bm \beta^t) \le f_{S_*}(\bm \beta^t)$ according to Lemma \ref{lemma:residual}.
When $\tau_t > \tau_*$, we have the following properties according to the constraint specified in equation \eqref{eq:tau_opt}.

\begin{align*}					
	r_{\tau}^2 \le \bigg(2 \cdot \frac{\tau r_{\tau_o}}{\tau_o} \bigg)^2 \ \stackrel{(a)}{\le}&\ \frac{64}{\tau'}\norm{\bm r_{S_* \cap S_t}}_2^2 \\
	\abs{S_t\setminus S_*}r_{\tau}^2 \stackrel{(b)}{\le}& 64(1-\gamma)\cdot \frac{n}{\tau'}\norm{\bm r_{S_* \cap S_t}}_2^2
\end{align*}

The inequality \textit{(a)} follows the definition of $\tau_o$ and the fact that $\abs{S_* \cap S_t} \ge \tau'$. The inequality \textit{(b)} follows $\abs{S_t\setminus S_*} \le (1-\gamma)\cdot n$ and $\norm{\bm r_{S_t\setminus S_*}}_2^2 \le \abs{S_t\setminus S_*}r_{\tau}^2$. Then we have
\begin{align*}					
		%
		f_{S_t\setminus S_*}(\bm \beta) \le& \Big[64(1-\gamma)\cdot \frac{n}{\tau'} + 1\Big]f_{S_*\setminus S_t}(\bm \beta) \\
		+& \Big[64(1-\gamma)\cdot \frac{n}{\tau'}\Big]f_{S_* \cap S_t}(\bm \beta) \\
		f_{S_t\setminus S_*}(\bm \beta) + f_{S_* \cap S_t}(\bm \beta) \stackrel{(c)}{\le}& \Big[64(1-\gamma)\cdot \frac{n}{\tau'} + 1\Big] f_{S_*}(\bm \beta) \\
		f_{S_t}(\bm \beta) \stackrel{(d)}{\le}& \Big[1+\frac{128(1-\gamma)}{2\gamma-1}\Big] f_{S_*}(\bm \beta)
\end{align*}

The inequality \textit{(c)} follows $f_{S_*}(\bm \beta) = f_{S_*\setminus S_t}(\bm \beta) + f_{S_* \cap S_t}(\bm \beta)$ and the inequality \textit{(d)} follows $\tau' = \tau_t-\frac{n}{2}$.
\end{proof}

\begin{theorem}\label{theorem:local_optima}
Assume that least squares function $f_S(\bm \beta) = \norm{\bm y_S - X_S^T \bm \beta}_2^2$ satisfies Subset Restricted Strong Convexity (SRSC) Property for $\forall \bm \beta_1, \bm \beta_2 \in \Omega_\mu$ and $\forall S \in \mathcal{S}_\gamma$, then we have
\begin{equation} \label{eq:theorem3}
\begin{aligned}
f_{\hat{S}}(\hat{\bm \beta}) - f_{S^*}(\bm \beta^*) \leq \frac{\alpha\lambda}{1+\alpha} f_{S^*}(\bm 0) + \bigg( \frac{\lambda}{1+\alpha} - 1 \bigg)f_{S^*}(\bm \beta^*) \\
\end{aligned}
\end{equation}
where $\alpha = \big(\frac{1}{\eta \cdot \varphi_\mu }\big)^2$ and $\lambda = \Big[1+\frac{128(1-\gamma)}{2\gamma-1}\Big]$. Specifically, when the uncorrupted set size is less than ground truth, $\lambda=1$.
\end{theorem}

\begin{proof}
	As function $f_S(\bm \beta)$ satisfies SRSC property, we have
	\begin{equation} \label{eq:theorem3}
	\begin{aligned}
	f_{S}(\bm \beta_1) - f_S(\bm \beta_2) \geq \langle \nabla f_S(\bm \beta_2), \bm \beta_1 - \bm \beta_2 \rangle + \frac{\varphi_\mu}{2}\norm{\bm \beta_1 - \bm \beta_2}_2^2 \\
	\end{aligned}
	\end{equation}
	for $\forall \bm \beta_1, \bm \beta_2 \in \Omega_\mu$ and $\forall S \in \mathcal{S}_\gamma$. Let $\supp(\bm \beta_1) = \Omega_{1}$, then we have	
	
\resizebox{.47\textwidth}{!}{\begin{minipage}{\linewidth}
\begin{equation}\label{eq:property1}
\begin{aligned}
f_{S}(\bm \beta_1)& - f_S(\bm \beta_2) \\
	\ge& \min_{\supp(\bm \beta) \subseteq \Omega_{1}} \Big\{ \langle \nabla f_S(\bm \beta_2), \bm \beta - \bm \beta_2 \rangle + \frac{\varphi_\mu}{2}\norm{\bm \beta - \bm \beta_2}_2^2 \Big\} \\
	\stackrel{(a)}{=}& -\frac{1}{2\varphi_\mu} {\Bigg\| \Big[ \nabla f_S(\bm \beta_2) \Big]_{\Omega_1} \Bigg\|}_2^2 \stackrel{(b)}{\ge} -\frac{\abs{\Omega_1}}{2\varphi_\mu} {\Bigg\| \Big[ \nabla f_S(\bm \beta_2) \Big]_{\Omega_1} \Bigg\|}_\infty^2\\
	\ge& -\frac{\mu}{2\varphi_\mu} {\Big\| \nabla f_S(\bm \beta_2) \Big\|}_\infty^2	
\end{aligned}
\end{equation}
\end{minipage}}

The equation (a) solves the minimum value of $\min_{\supp(\bm \beta) \subseteq \Omega_{1}}\{\cdot\}$ by setting its gradient to zero, and the inequality (b) follows $\abs{\Omega_1} \le \mu$. Let $\bm \beta_1$, $\bm \beta_2$ be ground truth coefficient $\bm \beta^*$ and estimated solution $\hat{\bm \beta}$ respectively, and set $S$ be ground truth uncorrupted set $S^*$, then we have
\resizebox{.47\textwidth}{!}{\begin{minipage}{\linewidth}
\begin{align*}
f_{S^*}(\hat{\bm \beta}) - f_{S^*}(\bm \beta^*) \le& \frac{\mu}{2\varphi_\mu} {\Big\| \nabla f_{S^*}(\hat{\bm \beta}) \Big\|}_\infty^2 \le& \frac{\mu}{2\varphi_\mu} {\Big( \frac{1}{\eta} \hat{\bm \beta}_{\min} \Big)}^2
\end{align*}
\end{minipage}}
where $\hat{\bm \beta}_{\min} = \min_i \abs{\hat{\bm \beta}_i}$. According to SRSC property, we have
{
\begin{align*}
f_{S^*}(\bm 0) - f_{S^*}(\hat{\bm \beta}) \geq \langle \nabla f_{S^*}(\hat{\bm \beta}), -\hat{\bm \beta} \rangle + \frac{\varphi_\mu}{2}\norm{\hat{\bm \beta}}_2^2
\end{align*}}

Because of $\langle \nabla f_{S^*}(\hat{\bm \beta}), -\hat{\bm \beta} \rangle \geq 0$, we have

\resizebox{.47\textwidth}{!}{\begin{minipage}{\linewidth}
\begin{align*}
\mu \hat{\bm \beta}_{\min}^2 \le \norm{\hat{\bm \beta}}^2_2 \le& \frac{2}{\varphi_\mu}\Big[ f_{S^*}(\bm 0) - f_{S^*}(\hat{\bm \beta}) \Big] \\
2\varphi_\mu \cdot \eta^2 \Big[ f_{S^*}(\hat{\bm \beta}) - f_{S^*}(\bm \beta^*) \Big] \stackrel{(c)}{\le}& \frac{2}{\varphi_\mu}\Big[ f_{S^*}(\bm 0) - f_{S^*}(\hat{\bm \beta}) \Big]
\end{align*}
\end{minipage}}

The inequality (c) follows the Equation \eqref{eq:property1}. Let $\alpha$ be $\big(\frac{1}{\eta \varphi_\mu}\big)^2$, then we have

\begin{align*}
f_{S^*}(\hat{\bm \beta}) - f_{S^*}(\bm \beta^*) \le \alpha \Big[ f_{S^*}(\bm 0) - f_{S^*}(\hat{\bm \beta}) \Big] \\
f_{S^*}(\hat{\bm \beta}) \le \frac{\alpha}{1+\alpha} \Big[ f_{S^*}(\bm 0) - f_{S^*}(\hat{\bm \beta}) \Big] + f_{S^*}(\bm \beta^*)
\end{align*}

According to Lemma \ref{lemma:tau}, we have

\begin{align*}
\frac{1}{\lambda} f_{\hat{S}}(\hat{\bm \beta}) \le f_{S^*}(\bm \beta^*) \le \frac{\alpha}{1+\alpha} \Big[ f_{S^*}(\bm 0) - f_{S^*}(\hat{\bm \beta}) \Big] + f_{S^*}(\bm \beta^*)\\
f_{\hat{S}}(\hat{\bm \beta}) - f_{S^*}(\bm \beta^*) \leq \frac{\alpha\lambda}{1+\alpha} f_{S^*}(\bm 0) + \bigg( \frac{\lambda}{1+\alpha} - 1 \bigg)f_{S^*}(\bm \beta^*)
\end{align*}

\end{proof}
Since the value of $f_{S^*}(\bm 0)$ and $f_{S^*}(\bm \beta^*)$ are both constants and $f_{S^*}(\bm \beta^*)$ is close to 0, the error bound of our solution is depended on the value of $\frac{\alpha\lambda}{1+\alpha}$. When the ratio of data corruption $\gamma$ is close to one, $\lambda$ is close to one according to its definition. In addition, the value of $\alpha$ is smaller when the parameter $\varphi_\mu$ of the SRSC property is smaller. Therefore, the error of our solution can be close to zero when both $\varphi_\mu$ and $\gamma$ are large enough.

\begin{figure*}[t]
	\centering
		\subfigure[p=2K, n=1K, $\mu$/$p$=20\%, dense]{%
			\label{fig:beta_1}
			\includegraphics[trim=0.6cm 0.1cm 0.6cm 0.1cm,width=0.32\linewidth]{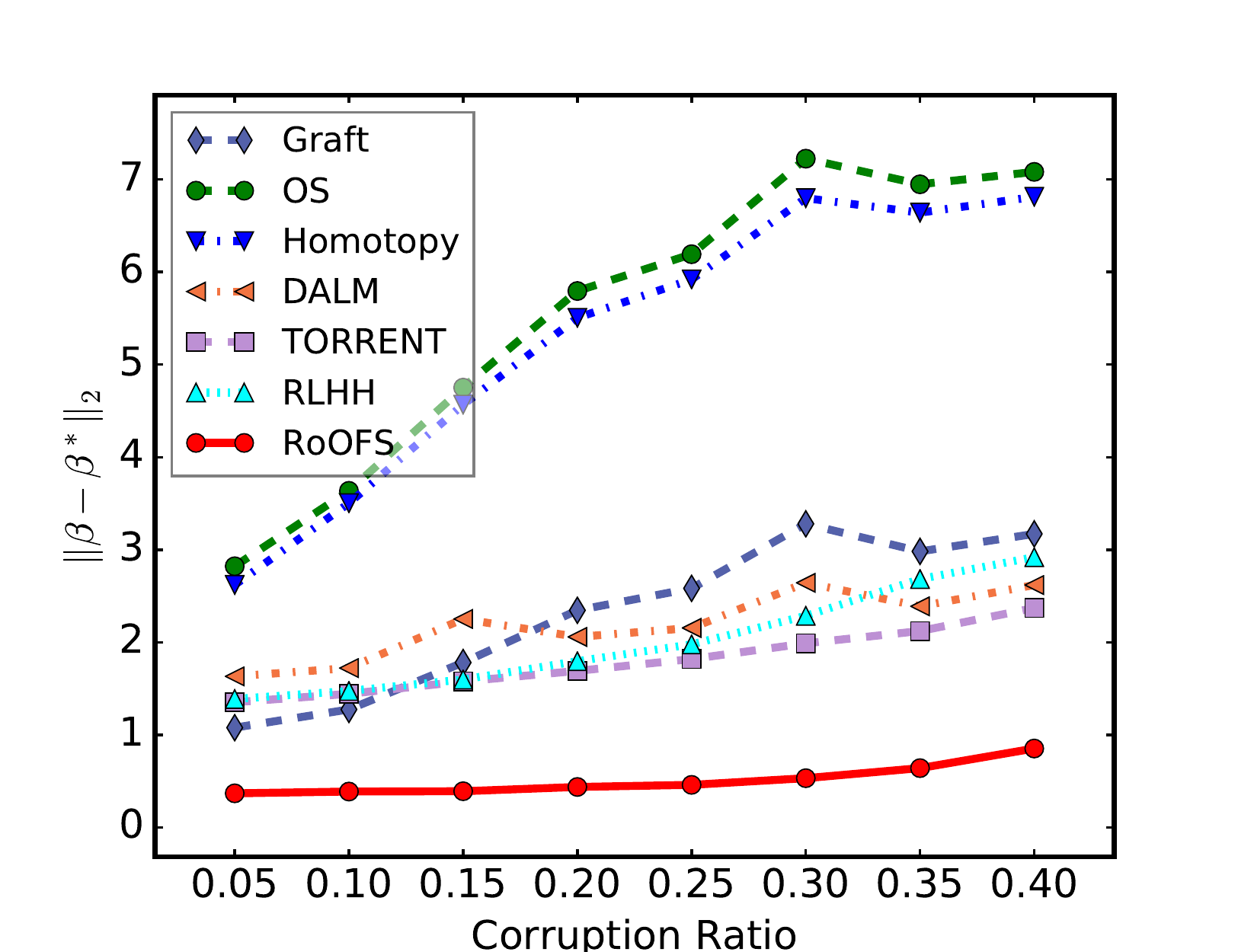}
	}
		\subfigure[p=4K, n=1K, $\mu$/$p$=20\%, dense]{%
			\label{fig:beta_2}
			\includegraphics[trim=0.6cm 0.1cm 0.6cm 0.1cm,width=0.32\linewidth]{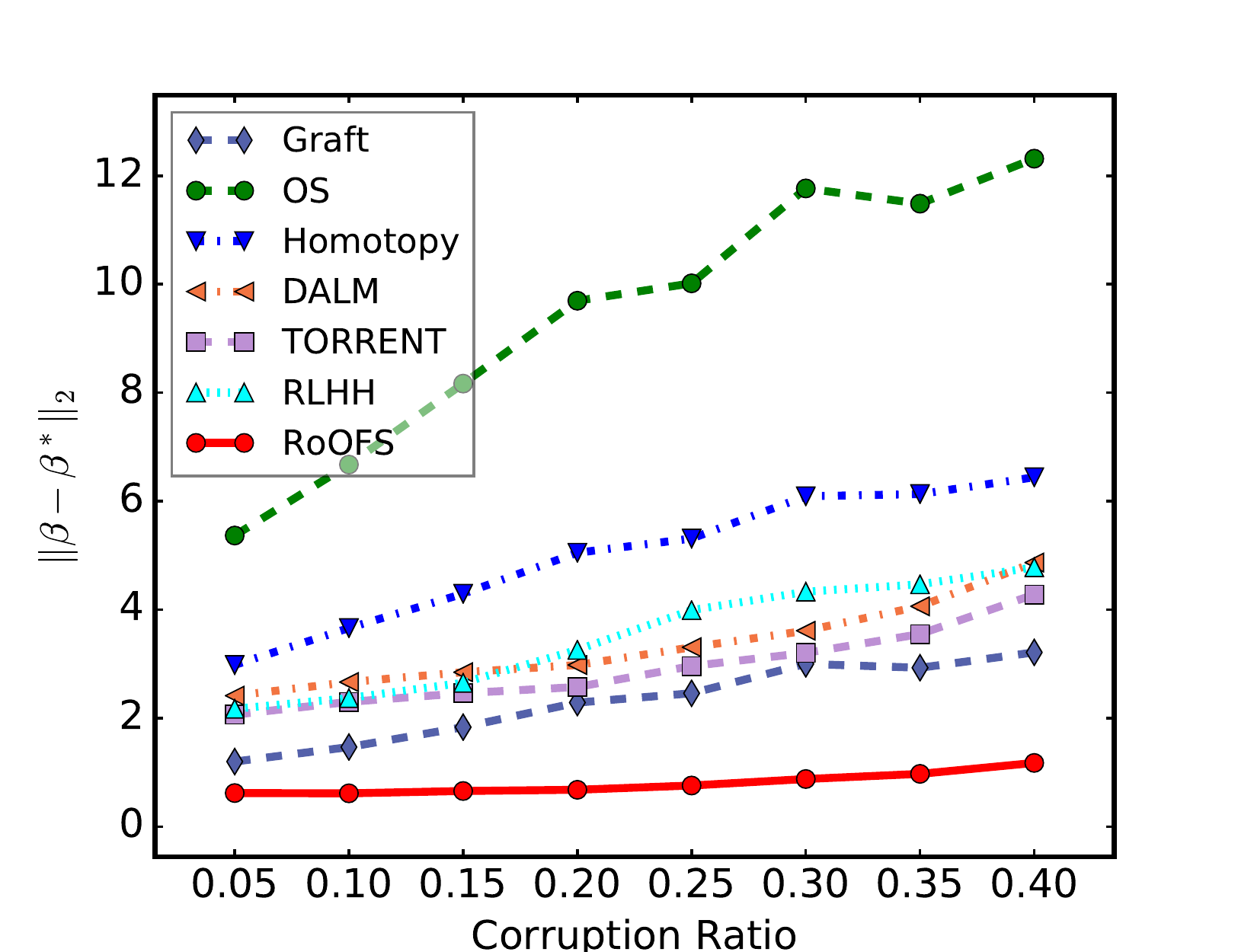}
	}
		\subfigure[p=4K, n=2K, $\mu$/$p$=20\%, dense]{%
			\label{fig:beta_3}
			\includegraphics[trim=0.6cm 0.1cm 0.6cm 0.1cm,width=0.32\linewidth]{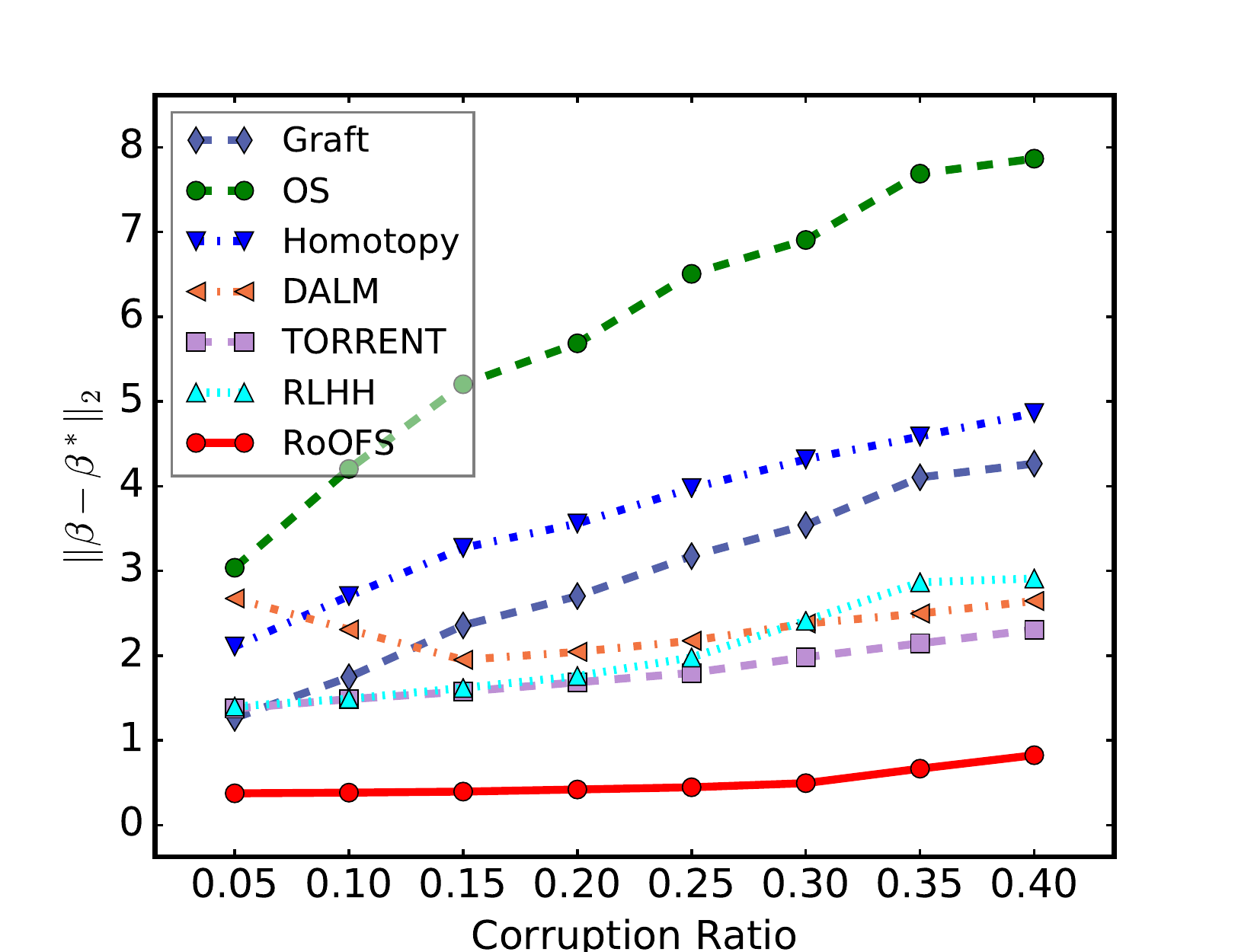}
	}

		\subfigure[p=4K, n=1K, $\mu$/$p$=40\%, dense]{%
			\label{fig:beta_4}
			\includegraphics[trim=0.6cm 0.1cm 0.6cm 0.1cm,width=0.32\linewidth]{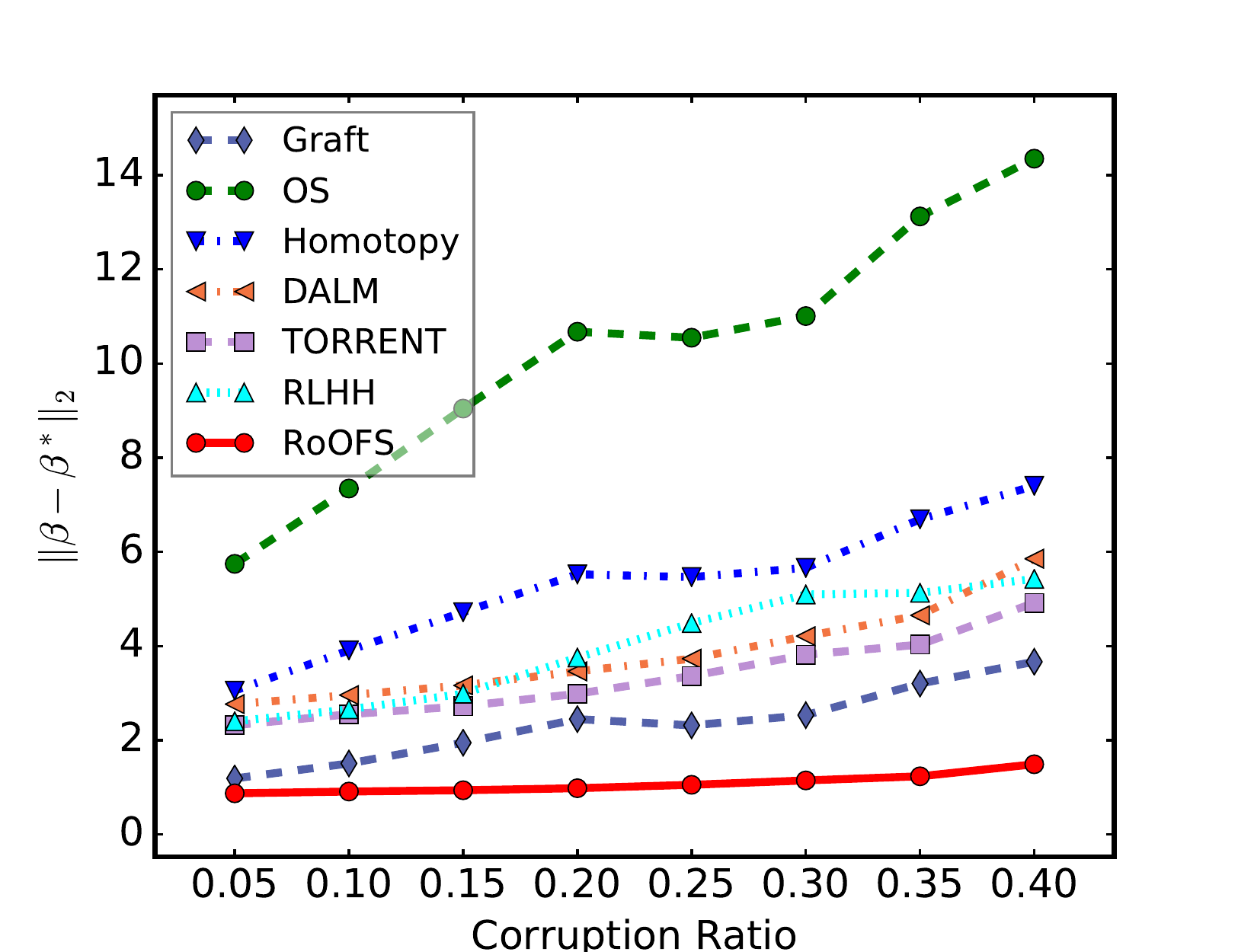}
	}
		\subfigure[p=4K, n=1K, $\mu$/$p$=80\%, dense]{%
			\label{fig:beta_5}
			\includegraphics[trim=0.6cm 0.1cm 0.6cm 0.1cm,width=0.32\linewidth]{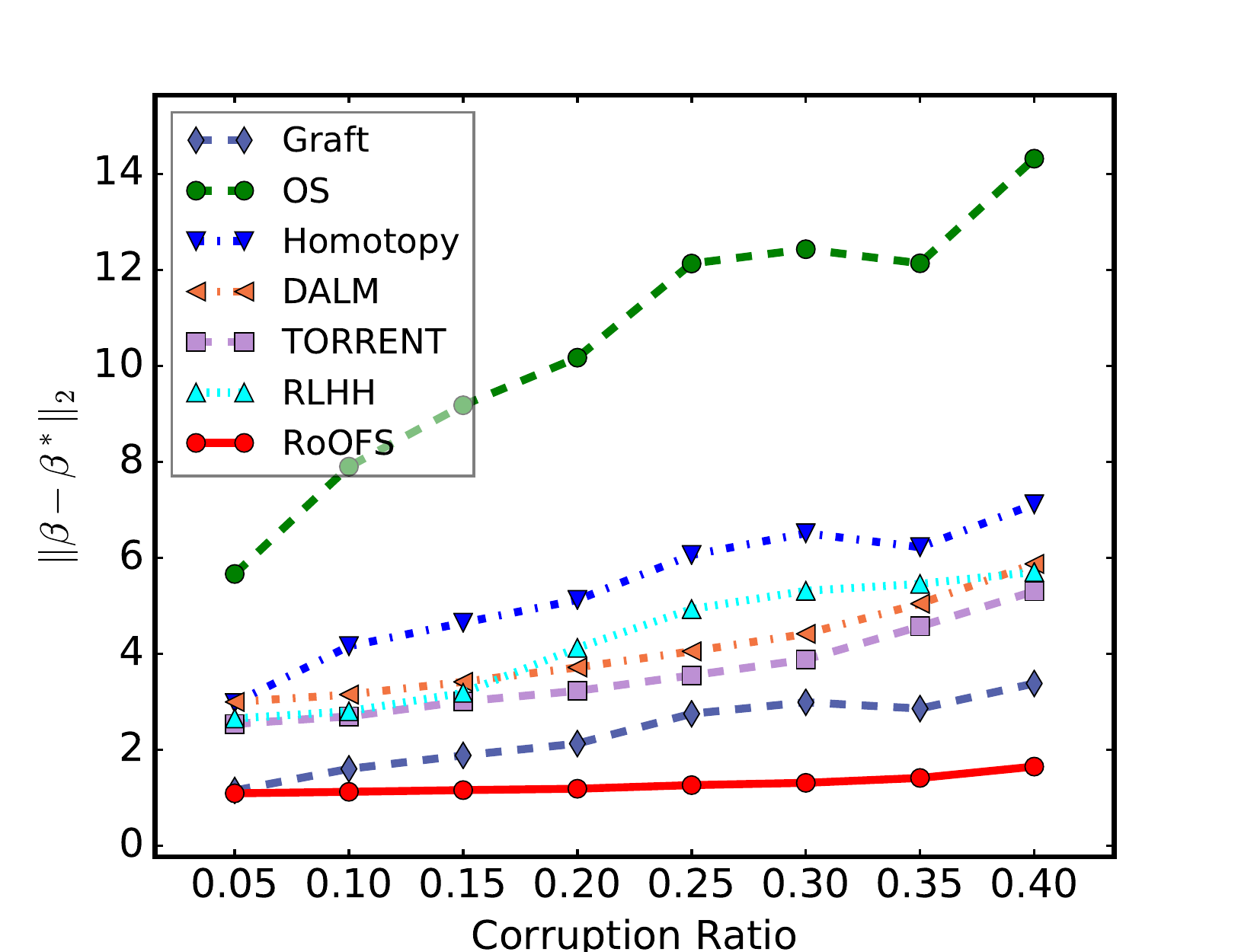}
	}
		\subfigure[p=2K, n=1K, $\mu$/$p$=20\%, no dense]{%
			\label{fig:beta_6}
			\includegraphics[trim=0.6cm 0.1cm 0.6cm 0.1cm,width=0.32\linewidth]{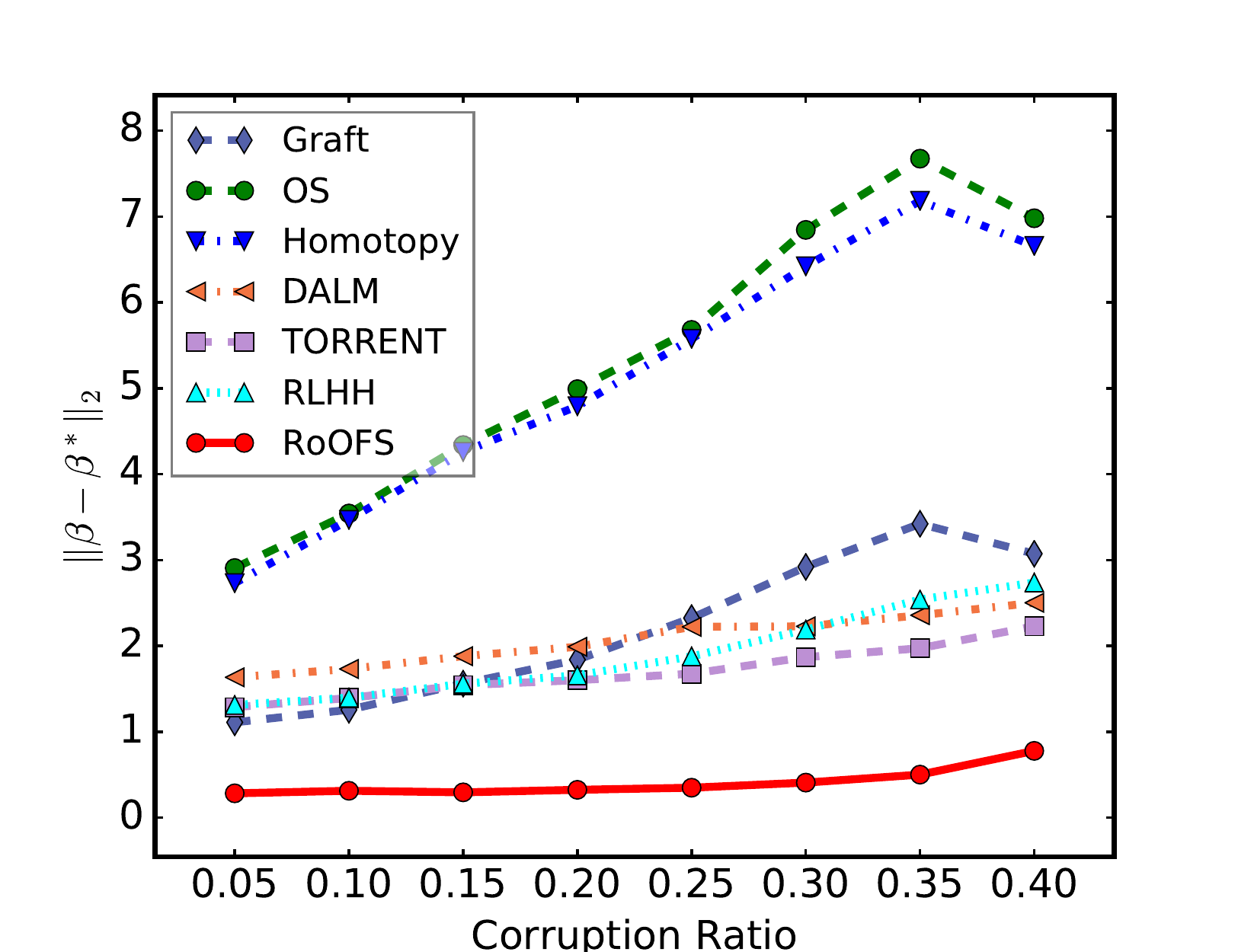}
	}
	
	\caption{
		\small Performance on regression coefficients recovery for different corruption ratios in uniform distribution.
	}%
	\label{fig:beta-cr}
\end{figure*}

\section{Experimental Results}\label{section:experiment}
In this section, we report the extensive experimental evaluation performed to verify the robustness, effectiveness of feature selection, and efficiency of the proposed method. All the experiments were conducted on a 64-bit machine with Intel(R) core(TM) quad-core processor (i7CPU@3.6GHz) and 32.0GB memory. Details of both the source code and sample data used in the experiment can be downloaded here\footnote{\url{https://goo.gl/C4HQjo}}.

\begin{table*}[h]
	\caption{F1 Scores for the Performance on Uncorrupted Set Recovery.}
	\centering
	\small
    \renewcommand\arraystretch{1.2}
	\label{table:uncorrupt_set}	
	\scalebox{1}{
		\begin{tabularx}{0.9\textwidth}{c *{13}{Y}}
			\toprule
			& \multicolumn{4}{c}{\textbf{p=2K, n=1K, $\mu$/$p$=20\% }}
			& \multicolumn{4}{c}{\textbf{p=2K, n=2K, $\mu$/$p$=20\%}}
			& \multicolumn{4}{c}{\textbf{p=4K, n=2K, $\mu$/$p$=20\%}} \\
			\cmidrule(lr){2-5} \cmidrule(lr){6-9} \cmidrule(lr){10-13}
			& 10\% & 20\% & 30\% & 40\% & 10\% & 20\% & 30\% & 40\% & 10\% & 20\% & 30\% & 40\%\\
			\midrule
			\textbf{Homotopy}	&0.980 & 0.912 & 0.849 & 0.682 & 0.977 & 0.923 & 0.854 & 0.834 &0.970 & 0.923 & 0.845 & 0.775\\	
			\textbf{DALM}		&0.976 & 0.915 & 0.865 & 0.825 & 0.973 & 0.921 & 0.885 & 0.924 &0.962 & 0.946 & 0.926 & 0.898\\	
			\textbf{TORR*}		&0.983 & 0.950 & 0.927 & 0.893 & 0.983 & 0.960 & 0.919 & 0.934 &0.978 & 0.954 & 0.934 & 0.916\\	
			\textbf{TORR25}		&0.965 & 0.899 & 0.842 & 0.762 & 0.961 & 0.909 & 0.828 & 0.770 &0.958 & 0.905 & 0.848 & 0.752\\				
			\textbf{RLHH}		&0.979 & 0.945 & 0.933 & 0.901 & 0.978 & 0.966 & 0.936 & 0.914 &0.980 & 0.959 & 0.940 & 0.896\\	
			\textbf{RoOFS}		&\textbf{0.991} & \textbf{0.986} & \textbf{0.974} & \textbf{0.933} & \textbf{0.993} & \textbf{0.991} & \textbf{0.976} & \textbf{0.946} &\textbf{0.993} & \textbf{0.988} & \textbf{0.975} & \textbf{0.923}\\		
			\bottomrule
	\end{tabularx}}	
	\scalebox{1}{
        \renewcommand\arraystretch{1.2}
		\begin{tabularx}{0.9\textwidth}{c *{13}{Y}}
			& \multicolumn{4}{c}{\textbf{p=2K, n=1K, $\mu$/$p$=60\%}}
			& \multicolumn{4}{c}{\textbf{p=2K, n=1K, $\mu$/$p$=20\% (nd)}}
			& \multicolumn{4}{c}{\textbf{p=4K, n=2K, $\mu$/$p$=20\% (nd)}} \\
			\cmidrule(lr){2-5} \cmidrule(lr){6-9} \cmidrule(lr){10-13}
			& 10\% & 20\% & 30\% & 40\% & 10\% & 20\% & 30\% & 40\% & 10\% & 20\% & 30\% & 40\%\\
			\midrule
			\textbf{Homotopy}	&0.979 & 0.932 & 0.829 & 0.708 &0.972 & 0.923 & 0.853 & 0.717 &0.985 & 0.913 & 0.868 & 0.789\\	
			\textbf{DALM}		&0.975 & 0.939 & 0.863 & 0.826 &0.965 & 0.910 & 0.886 & 0.842 &0.984 & 0.951 & 0.937 & 0.889\\	
			\textbf{TORR*}		&0.979 & 0.957 & 0.937 & 0.870 &0.974 & 0.950 & 0.935 & 0.896 &0.988 & 0.960 & 0.947 & 0.904\\	
			\textbf{TORR25}		&0.952 & 0.912 & 0.833 & 0.690 &0.952 & 0.911 & 0.859 & 0.758 &0.968 & 0.908 & 0.864 & 0.745\\	
			\textbf{RLHH}		&0.975 & 0.959 & 0.928 & 0.845 &0.973 & 0.959 & 0.935 & 0.907 &0.983 & 0.965 & 0.940 & 0.912\\	
			\textbf{RoOFS}		&\textbf{0.982} & \textbf{0.984} & \textbf{0.962} & \textbf{0.910} &\textbf{0.989} & \textbf{0.993} & \textbf{0.985} & \textbf{0.947} &\textbf{0.994} & \textbf{0.991} & \textbf{0.988} & \textbf{0.933}\\
			\bottomrule
	\end{tabularx}}	
\end{table*}

\subsection{Datasets and Metrics}
To demonstrate the performance of our proposed method, comprehensive experiments are performed in synthetic datasets whose simulation samples were randomly generated according to the model in Equation \eqref{eq:model}. Specifically, we sample the regression coefficients $\bm \beta^* \in \mathbbm{R}^p$ as a random unit norm vector with feature ratio constraint $\norm{\bm \beta}_0 = \mu$. The data matrix $X$ was drawn independently and identically distributed from $\bm x_i \sim \mathcal{N}(\bm 0, I_p)$ and the uncorrupted response variables were generated as $y_i^* = \bm x_i^T \bm \beta^*$. The set of uncorrupted samples $S$ was selected as a uniformly random $\tau_*$-sized subset of $[n]$. The response vector $\bm y$ containing corrupted samples was generated as $\bm y = \bm y^* + \bm u +\bm \varepsilon$, where the corruption vector $\bm u$ was sampled from the uniform distribution $\big [-5\|\bm y^*\|_\infty, 5\|\bm y^*\|_\infty \big]$ and the additive dense noise was $\varepsilon_i \sim \mathcal{N}(0, \sigma^2)$.
For the real-world data set, we applied our methods on the IMDb reviews data set for the review score prediction. The data set contains 50,000 popular movie reviews with the review score from 1 to 10 provided by the IMDb website. The adversarial data corruption vector $\bm u$ was appended to its original review score, where $\bm u$ was also sampled from the range $\big [-5\|\bm y^*\|_\infty, 5\|\bm y^*\|_\infty \big]$ randomly.

Following the setting in \cite{bhatia2015robust}\cite{ijcai2017-480}, we measured the performance of the regression coefficients recovery using the standard $L_2$ error $e = \norm{\hat{\bm \beta} - \bm \beta^*}_2$, where $\hat{\bm \beta}$ represents the recovered coefficients for each method and $\bm \beta^*$ is the true regression coefficients. To validate the performance for corrupted set discovery, the F1 score is measured by comparing the discovered corrupted sets with the actual ones. Similarly, the F1 score is also used to measure the effectiveness of feature selection by comparing the selected feature set with actual ones. To compare the scalability of each method, the CPU running time for each of the competing methods was also measured.

\begin{figure*}[ht]
	\centering
		\subfigure[p=2K, corruption ratio=20\%]{%
			\label{fig:beta-fr_1}
			\includegraphics[trim=0.6cm 0.1cm 0.6cm 0.1cm,width=0.32\linewidth]{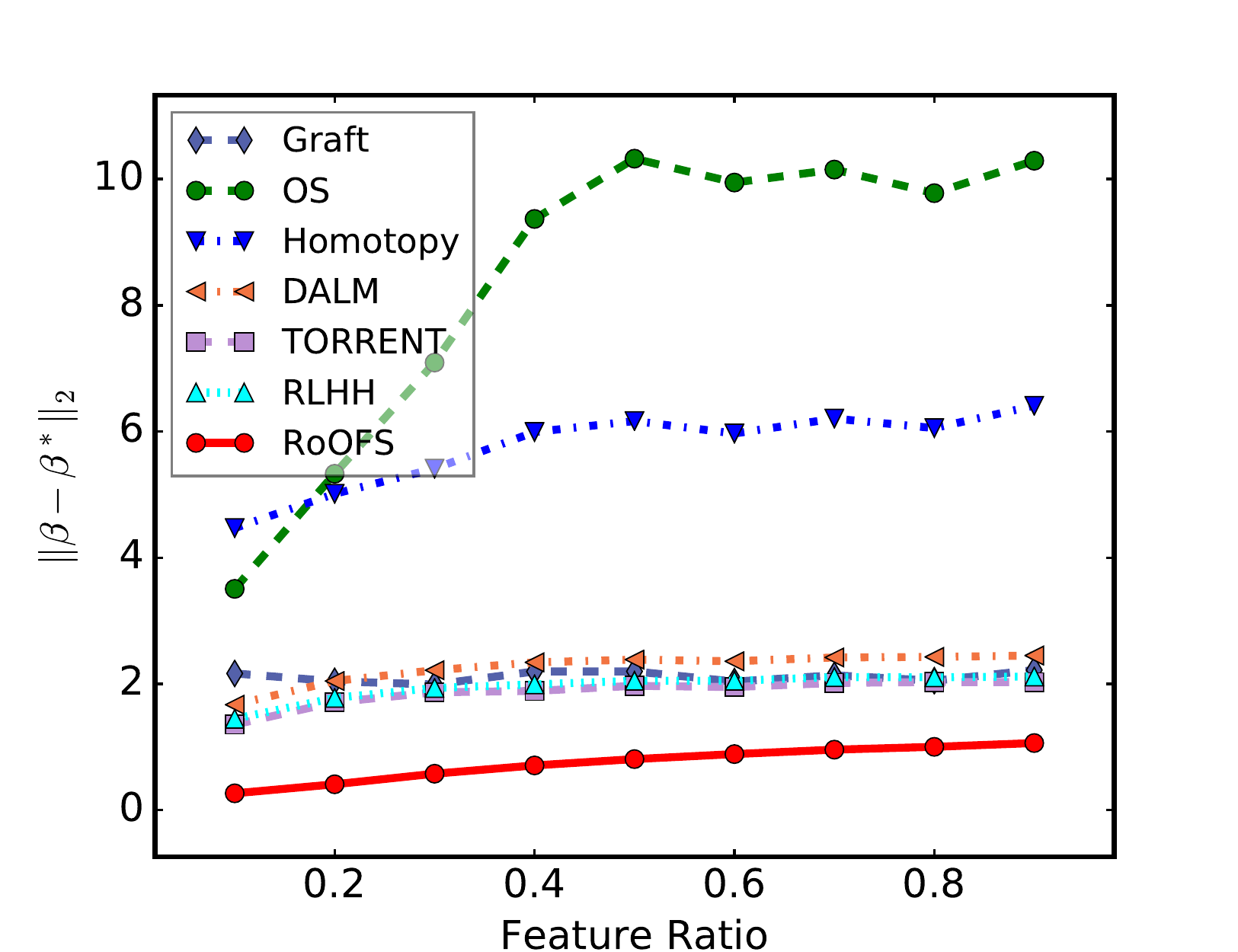}
	}
		\subfigure[p=4K, corruption ratio=20\%]{%
			\label{fig:beta-fr_2}
			\includegraphics[trim=0.6cm 0.1cm 0.6cm 0.1cm,width=0.32\linewidth]{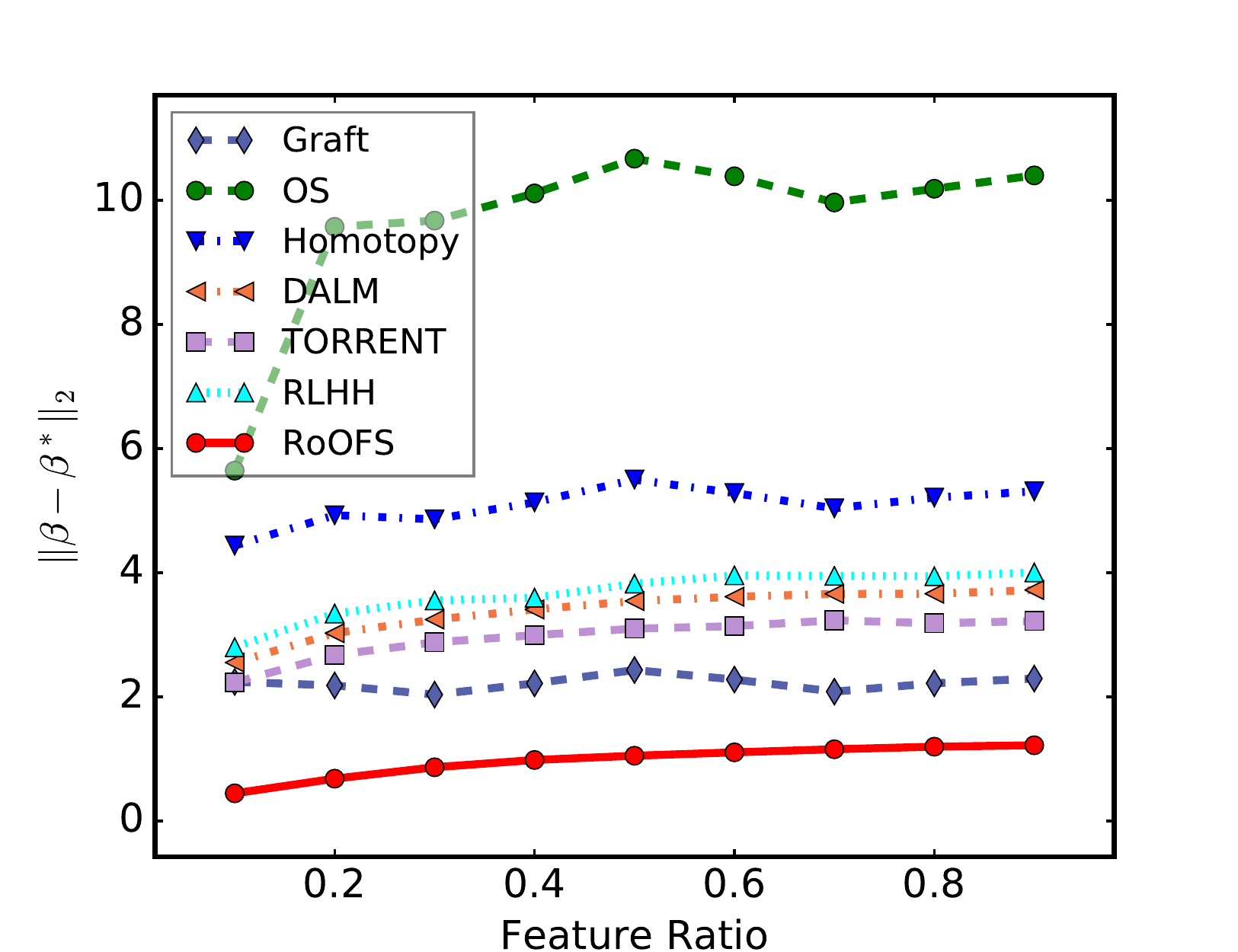}
	}
		\subfigure[p=2K, corruption ratio=40\%]{%
			\label{fig:beta-fr_3}
			\includegraphics[trim=0.6cm 0.1cm 0.6cm 0.1cm,width=0.32\linewidth]{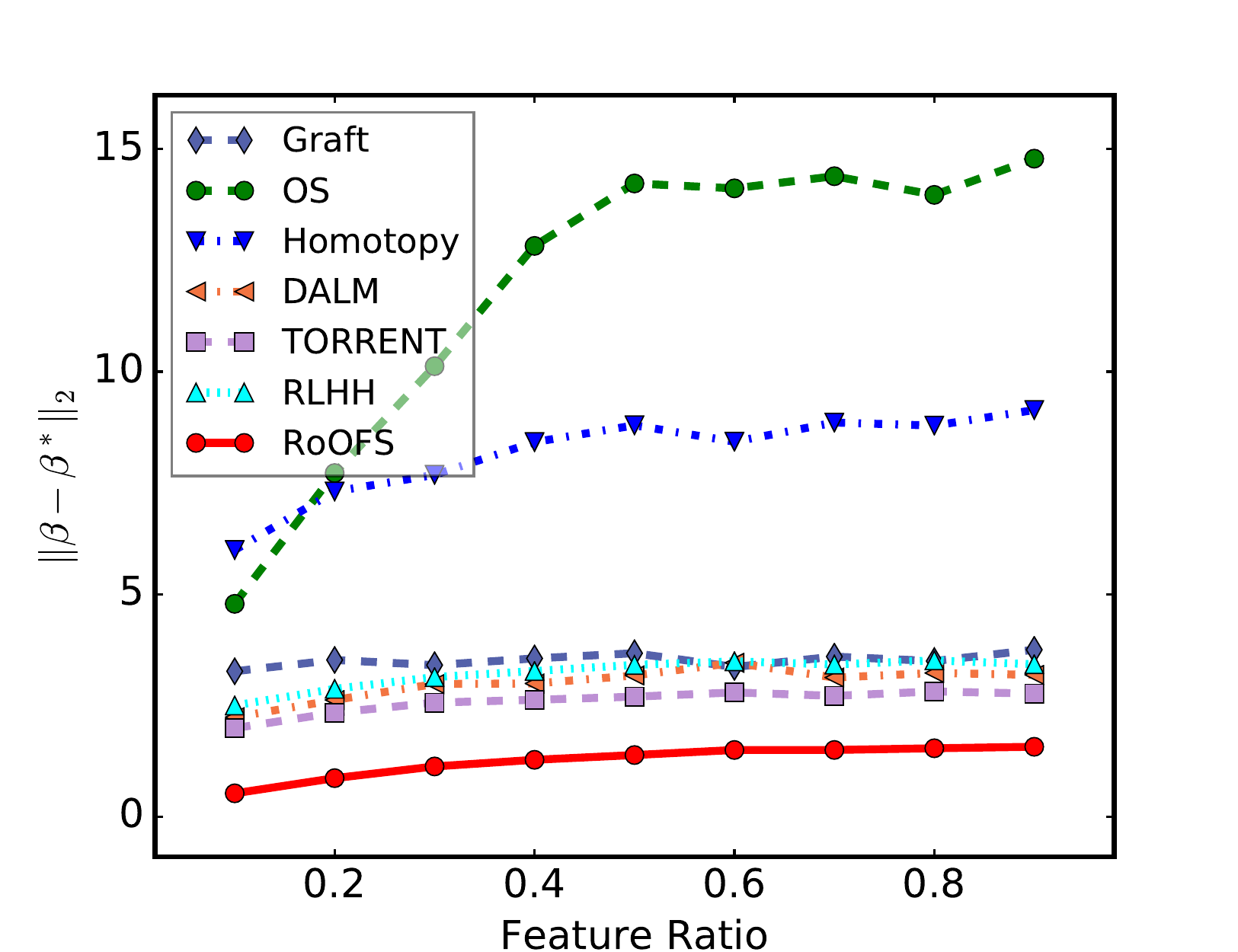}
	}
	
	\caption{%
		\small Performance on regression coefficients recovery for different ratios of feature ratio (n=1K, dense noise).
	}
	\label{fig:beta-fr}
\end{figure*}

\subsection{Comparison Methods}
The following methods are included in the performance comparison presented here: \textit{Grafting} \cite{Perkins:2003:OFS:3041838.3041913}. The \textit{Grafting} method is an online version of $L_1$ regularization approach to selects features. \textit{Online Substitution} (\textit{OS}) \cite{Yang:2016:OFS:2939672.2939881} is a parameter-free online feature selection algorithm with limited-memory. Both \textit{Grafting} and \textit{OS} cannot handle the adversarial data corruption and train models without considering data corruption. We also compared our method to the robust regression methods \cite{Wright:2010:DEC:1840493.1840533} \cite{nguyen2013exact}. \textit{Homotopy} and \textit{DALM} are two $L_1$ based solvers that outperform other $L_1$ methods both in terms of recovery properties and running time \cite{Yang:EECS-2010-13}. A hard thresholding method, \textit{TORRENT (abbreviated "TORR")} \cite{bhatia2015robust}, developed for robust regression was also compared to our method. As the method requires a parameter for the corruption ratio, which is difficult to estimate in practice, we chose two versions of parameter settings: \textit{TORR*} and \textit{TORR25}. \textit{TORR*} uses the true corruption ratio as its parameter, and \textit{TORR25} applies parameter that is uniformly distributed across the range of $\pm 25\%$ off the true value. Another recently proposed heuristic hard thresholding method, \textit{RLHH} \cite{ijcai2017-480}, is also compared in our experiment. The method is a parameter-free approach, where the data corruption is estimated by a heuristic hard thresholding method. As all these robust methods are not designed for online feature selection, we run them individually in different feature batches and select features with largest $\mu$ weights in regression coefficients when $\norm{\bm \beta}_0 = \mu$. 

\subsection{Recovery of regression coefficients}
We selected 6 competing methods with which to evaluate the recovery performance of regression coefficients $\bm \beta$: \textit{Grafting}, \textit{OS}, \textit{Homotopy}, \textit{DALM}, \textit{TORR}, \textit{RLHH}. 
Figures \ref{fig:beta_1} and \ref{fig:beta_2} show the recovery performance for different feature numbers when the data size is fixed. The results show that 1) the proposed method, \textit{RoOFS}, outperforms all the competing methods in all the setting of corruption ratios, and 2) The performance of \textit{RoOFS} is very resistant to the corruption data because the error of \textit{RoOFS} method increases much more slowly than others when corruption ratio increases from 5\% to 40\%. Figures \ref{fig:beta_2} and \ref{fig:beta_3} show that when data size increases, we have similar conclusion on the performance except the overall error is decreased since more data is applied. Figures \ref{fig:beta_4} and \ref{fig:beta_5} show that the result of coefficient recovery remains the same when the number of selected features increase. Figure \ref{fig:beta_6} shows that almost all the methods without the dense noise setting perform more than 50\% better than that in dense noise settings. Specifically, the error of \textit{RoOFS} is close to zero which means it can almost exactly recover the ground true regression coefficients without the dense noise setting.

Figure \ref{fig:beta-fr} shows the recovery performance of regression coefficients in different ratios of feature sparsity. In general, the performance of \textit{RoOFS} method outperforms all the other competing methods in all the data settings. Figure \ref{fig:beta-fr_1} and \ref{fig:beta-fr_2} show that 1) when the feature ratio increases, the recovery error of \textit{RoOFS} method grows linearly with a small slope, which means our approach can be fitted into different settings of feature sparsity, and 2) the \textit{RoOFS} method performs constantly well when the feature number increases from 2K to 4K. In addition, Figure \ref{fig:beta-fr_1} and \ref{fig:beta-fr_3} show that the \textit{RoOFS} method is robust to corrupted data, because the error is not significantly impacted when the corruption ratio increases from 20\% to 40\%.

\subsection{Recovery of Uncorrupted Set}
As the online feature selection methods \textit{Grafting} and \textit{OS} do not explicitly estimate uncorrupted sets, we compared our proposed method with the robust methods: \textit{Homotopy}, \textit{DALM}, \textit{TORR}, and \textit{RLHH}. For the \textit{TORR} algorithm, we use two parameter settings of \textit{TORR*} and \textit{TORR25} for 0\% and 25\% deviation of true corrupted ratio, respectively. Table \ref{table:uncorrupt_set} shows the following: 1) \textit{RoOFS} outperforms all the other methods up to 14.9\% in different settings of data sizes, feature numbers and ratios of feature sparsity. 2) When increasing the corruption ratio, the F1 scores decrease for all the methods. Also, the F1 scores slightly increase 0.5\% in average when data size become two times larger, which indicates the number of features has few influence on the estimation of uncorrupted set. 3) The result of \textit{TORR} methods is highly dependent on the corruption ratio parameter: the results of \textit{TORR*} is up to 26.1\% better than \textit{TORR25}. It is important to note that the true corruption ratio parameter used in \textit{TORR*} cannot be estimated exactly in practice. 4) Without the dense noise settings, the F1 scores increase less than 1\% compared to the F1 score based on the same setting with dense noise, which shows that dense noise has small impact on the performance of uncorrupted set recovery.

\begin{figure}[t]
	\centering	
	\scalebox{1.16}{
		\subfigure[p=2K, $\mu$/$p$=20\%, cr=10\%]{%
			\label{fig:runtime-ds}
			\includegraphics[trim=0.6cm 0.1cm 0.6cm 0.1cm,width=0.7\linewidth]{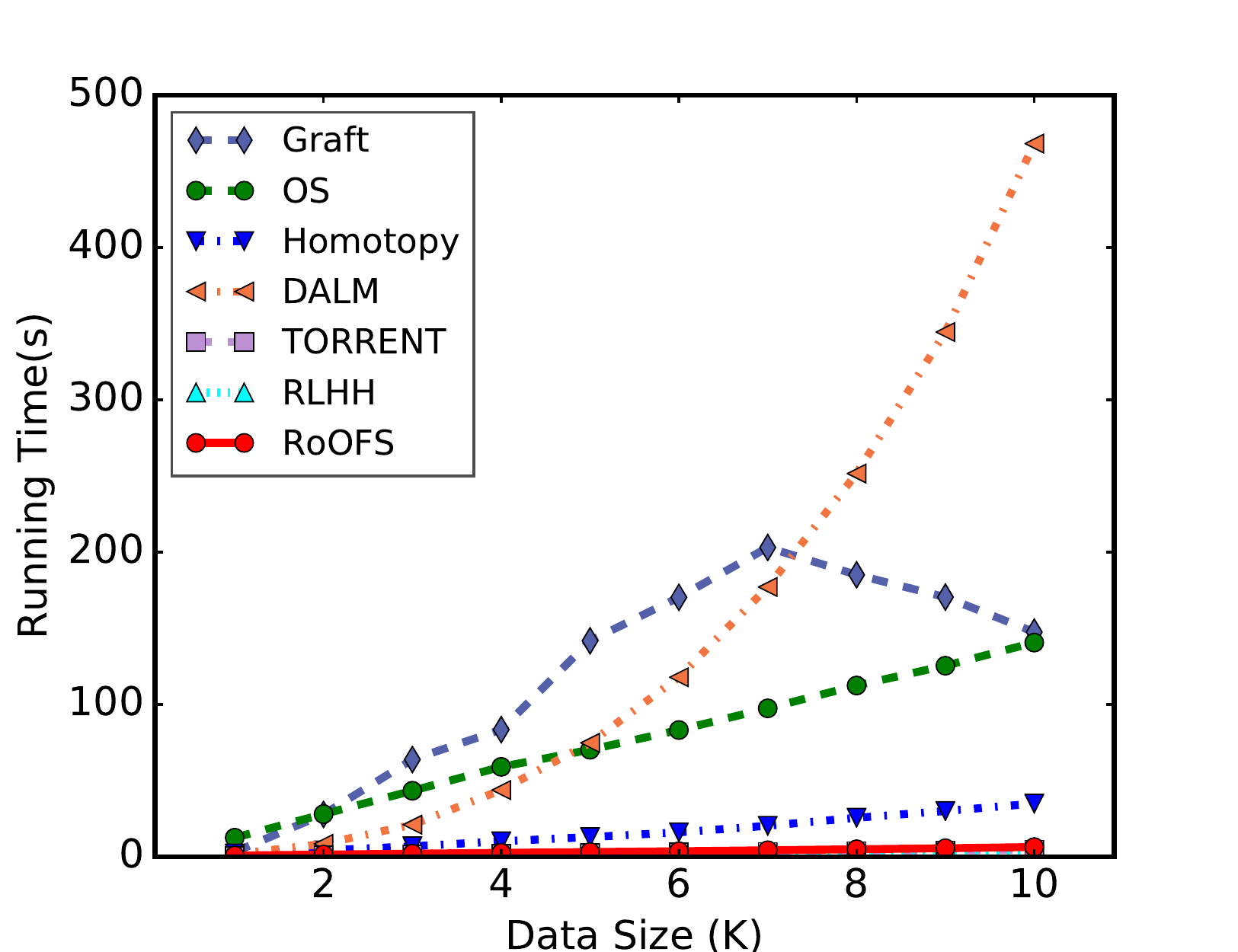}
	}}

	\scalebox{1.16}{
		\subfigure[n=1K, $\mu$/$p$=20\%, cr=10\%]{%
			\label{fig:runtime-fn}
			\includegraphics[trim=0.6cm 0.1cm 0.6cm 0.1cm,width=0.7\linewidth]{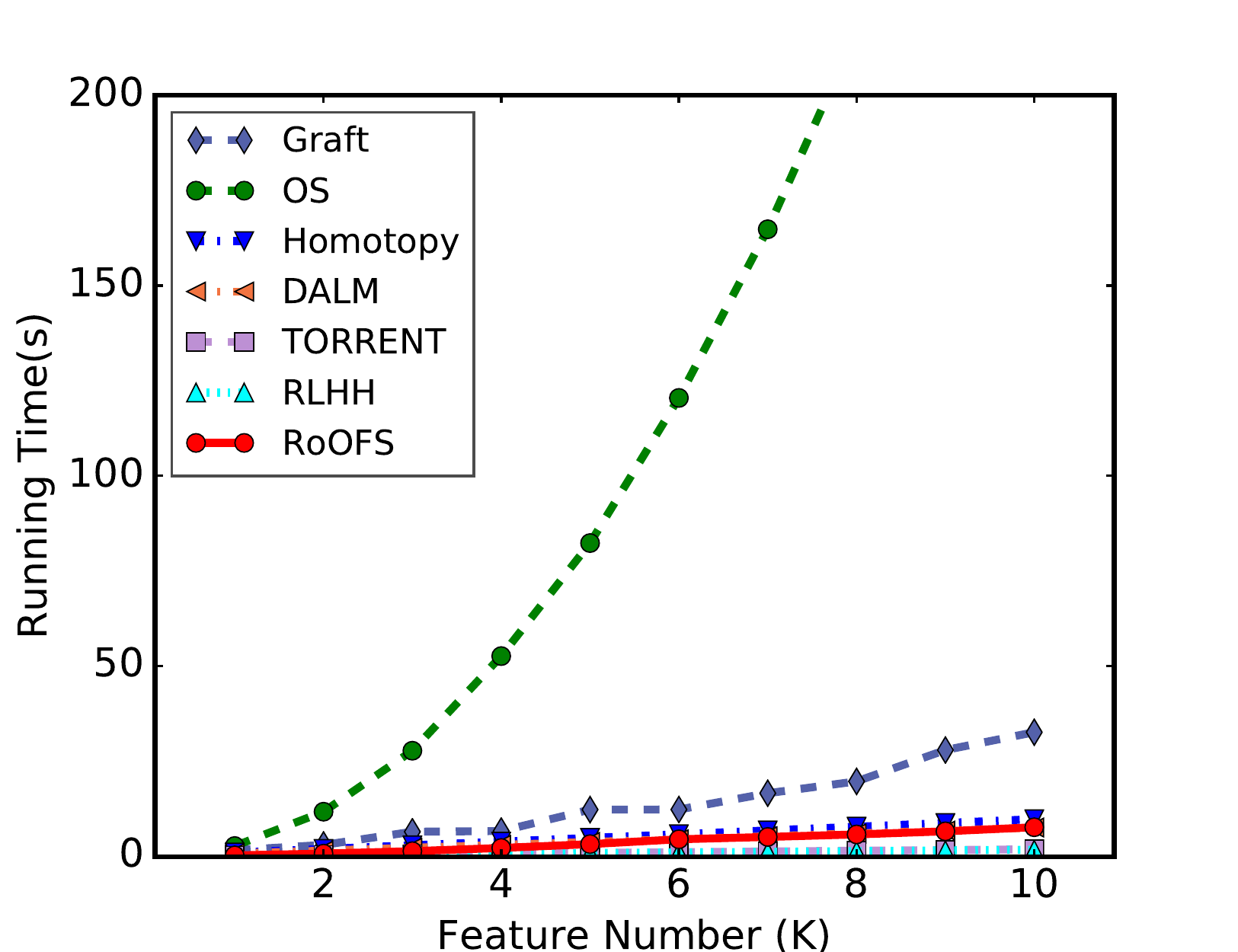}
	}}
	
	\caption{
		\small Running time for different data and feature sizes.
	}%
	\label{fig:runtime}
\end{figure}

\subsection{Performance of Feature Selection}
We selected all the six competing methods to evaluate the performance of feature selection in different settings including data sizes, feature numbers, and dense noises. For each data setting, we chose different ratios of feature sparsity (also known as $\mu$/$p$) ranging from 10\% to 60\%. Table \ref{table:feature_selection} shows the following: 1) the F1 scores of \textit{RoOFS} method is up to 69.2\% better than other methods, especially when the feature ratio is less than 40\%. 2) Although the F1 scores of most methods such as \textit{Grafting} and \textit{TORR} are above 0.6 when the ratio is larger than 50\%, the performance degraded significantly when the ratio decreased to 10\%. However, the F1 score of \textit{RoOFS} method is constantly higher than 0.85 in all the ratios of features. 3) \textit{OS} method is very competitive in the task of feature selection; however, it still has lower F1 scores in all the settings when the ratio is less than 50\%. 4) The setting of dense noise does not have significant impact on the performance of feature selection, since the F1 score without dense noise is only less than 1\% larger than that in dense noise setting.

\begin{table*}[t]
	\caption{F1 Score on Performance of Feature Selection (cr=30\%).}
	\centering
	\small
    \renewcommand\arraystretch{1.2}
	\label{table:feature_selection}	
		\begin{tabularx}{0.9\textwidth}{c *{13}{Y}}
			\toprule
			& \multicolumn{6}{c}{\textbf{p=10K, n=10K, dense }}
			& \multicolumn{6}{c}{\textbf{p=20K, n=10K, dense}} \\
			\cmidrule(lr){2-7} \cmidrule(lr){8-13}
			& 10\% & 20\% & 30\% & 40\% & 50\% & 60\% & 10\% & 20\% & 30\% & 40\% & 50\% & 60\%\\
			\midrule
			\textbf{Grafting}		&0.130 & 0.201 & 0.302 & 0.543 & 0.759 & 0.789 & 0.111 & 0.399 &0.642 & 0.773 & 0.844 & 0.895\\	
			\textbf{OS}			&0.706 & 0.649 & 0.626 & 0.643 & 0.810 & 0.678 & 0.611 & 0.611 & 0.606 & 0.689 & 0.836 & \textbf{0.975}\\	
			\textbf{Homotopy}	&0.116 & 0.217 & 0.304 & 0.406 & 0.498 & 0.667 & 0.109 & 0.202 & 0.417 & 0.625 & 0.750 & 0.833\\	
			\textbf{DALM}		&0.130 & 0.219 & 0.309 & 0.418 & 0.516 & 0.597 & 0.114 & 0.215 & 0.390 & 0.404 & 0.502 & 0.603\\	
			\textbf{TORR}		&0.275 & 0.297 & 0.368 & 0.446 & 0.525 & 0.647 & 0.320 & 0.338 & 0.395 & 0.458 & 0.535 & 0.623\\	
			\textbf{RLHH}		&0.193 & 0.261 & 0.338 & 0.416 & 0.510 & 0.647 & 0.322 & 0.336 & 0.390 & 0.461 & 0.536 & 0.628\\	
			\textbf{RoOFS}		&\textbf{0.911} & \textbf{0.910} & \textbf{0.876} & \textbf{0.870} & \textbf{0.895} & \textbf{0.891} & \textbf{0.876} & \textbf{0.842} &\textbf{0.832} & \textbf{0.848} & \textbf{0.881} & 0.906\\	

            \bottomrule
	   \end{tabularx}
        \renewcommand\arraystretch{1.2}
		\begin{tabularx}{0.9\textwidth}{c *{13}{Y}}
			& \multicolumn{6}{c}{\textbf{p=10K, n=5K, dense}}
			& \multicolumn{6}{c}{\textbf{p=10K, n=10K, no dense}} \\
			\cmidrule(lr){2-7} \cmidrule(lr){8-13}
			& 10\% & 20\% & 30\% & 40\% & 50\% & 60\% & 10\% & 20\% & 30\% & 40\% & 50\% & 60\%\\
			\midrule
			\textbf{Grafting}		& 0.114 & 0.206 & 0.426 & 0.641 & 0.765 & 0.836 & 0.142 & 0.224 & 0.293 & 0.527 & 0.698 & 0.797\\	
			\textbf{OS}			& 0.657 & 0.596 & 0.618 & 0.688 & 0.840 & \textbf{0.961} & 0.688 & 0.682 & 0.647 & 0.649 & 0.655 & 0.688\\	
			\textbf{Homotopy}	& 0.107 & 0.207 & 0.304 & 0.395 & 0.500 & 0.667 & 0.126 & 0.203 & 0.307 & 0.405 & 0.500 & 0.667\\	
			\textbf{DALM}		& 0.113 & 0.210 & 0.311 & 0.396 & 0.504 & 0.602 & 0.140 & 0.226 & 0.308 & 0.407 & 0.504 & 0.609\\	
			\textbf{TORR}		& 0.312 & 0.336 & 0.391 & 0.461 & 0.532 & 0.627 & 0.475 & 0.448 & 0.472 & 0.521 & 0.579 & 0.646\\	
			\textbf{RLHH}		& 0.314 & 0.334 & 0.388 & 0.463 & 0.530 & 0.624 & 0.476 & 0.458 & 0.487 & 0.531 & 0.585 & 0.646\\	
			\textbf{RoOFS}		& \textbf{0.873} & \textbf{0.830} &\textbf{0.837} & \textbf{0.859} & \textbf{0.883} & 0.909 & \textbf{0.917} & \textbf{0.922} &\textbf{0.898} & \textbf{0.900} & \textbf{0.889} & \textbf{0.900}\\		
			\bottomrule
	\end{tabularx}
\end{table*}

\subsection{Performance in real-world data}
To evaluate the robustness of our proposed methods in a real-world dataset, we compared the performance of sentiment prediction in different corruption settings, ranging from 5\% to 40\%. The dataset was first proposed by Maas et al. \cite{Maas2011Learning} as a benchmark for sentiment analysis. It consists of movie reviews taken from IMDB. One key aspect of this dataset is that each movie review has several sentences. The 100,000 movie reviews are divided into three datasets: 25,000 labeled training instances, 25,000 labeled test instances and 50,000 unlabeled training instances. The unlabeled data were designed as the additional corruption to the dataset: the score of sentiment were random number between one to ten.
Table \ref{table:sentiment prediction} shows the mean absolute error of sentiment prediction in the IMDB datasets.
From the result, we can conclude: 
1) \textit{RoOFS} method outperform all the other methods in different corruption settings. 
2) Although the absolute error of the other methods such as \textit{Homotopy} and \textit{DALM} are above 4, the performance varied significantly when the ratio changed because these methods highly dependent on the parameters and it's hard to estimate the feature sparsity ratio and true corruption ratio in the real-world data. However, the performance of \textit{RoOFS} method is constantly above 3.10 in all the ratios of corruption.
3) It is true that \textit{OS} has a very competitive performance in all the corruption settings because the deviation of corruption is small, which is less than 50\% from the labeled data. But the running time of \textit{OS} is too high to train the data which has 10k features.
4) When increasing the corruption ratio, the absolute error of \textit{RoOFS} method decreased. 

\begin{table}[t]
	\caption{Mean Absolute Error of Sentiment Prediction.}
	\centering
	\small
    \renewcommand\arraystretch{1.2}
	\label{table:sentiment prediction}	
		\begin{tabularx}{0.96\linewidth}{c *{13}{Y}}
			\toprule
			& \multicolumn{6}{c}{\textbf{p=10K, n=10K}}\\
			\cmidrule(lr){2-7} \cmidrule(lr){8-13}
			& 5\% & 10\% & 20\% & 30\% & 40\% & Avg \\
			\midrule
			\textbf{Grafting}		&3.162 & 3.162 & 3.162 & 3.162 & 3.162 & 3.162 \\	
			\textbf{OS}			&3.111 & 3.078 & 3.078 & 3.113 & 3.108 & 3.098 \\	
			\textbf{Homotopy}	&4.157 & 3.925 & 3.894 & 3.828 & 3.540 & 3.869 \\	
			\textbf{DALM}		&3.573 & 3.311 & 3.523 & 3.459 & 3.252 & 3.424 \\	
			\textbf{TORR}		&5.090 & 3.928 & 4.596 & 5.147 & 4.218 & 4.596 \\	
			\textbf{RLHH}		&5.448 & 4.198 & 3.716 & 4.269 & 4.626 & 4.451 \\	
			\textbf{RoOFS}		&\textbf{3.074} & \textbf{3.073} & \textbf{3.072} & \textbf{3.070} & \textbf{3.066} & \textbf{3.071} \\	

            \bottomrule
	   \end{tabularx}
\end{table}

\subsection{Efficiency}
To evaluate the efficiency of our proposed method, we compared the performances of all the competing methods for two difference settings: data sizes and feature numbers. For \textit{Grafting} and \textit{OS} methods, the online features are handled individually due to their design. For the other methods, a hundred features are handled together as a batch. As Figure \ref{fig:runtime} shows, we found the following: 1) \textit{RoOFS} algorithm has a very competitive efficiency compared to the thresholding based methods, \textit{TORR} and \textit{RLHH}, and significantly outperforms other four methods. 2) The running time of \textit{RoOFS} algorithm increases linearly when both data size and feature number increase, which indicates that our algorithm can be scaled to massive datasets. 3) The running time of \textit{DALM} method increases exponentially when data size increases, however, its efficiency has rarely impacted by increasing the number of features. 4) The efficiency of \textit{Grafting} method fluctuates largely on the different data sizes, which indicates that its running time depends on the data size and content of data.

\section{Conclusion}\label{section:conclusion}
In this paper, a novel robust regression algorithm via online feature selection, \textit{RoOFS}, is proposed to recover the regression coefficients and the uncorrupted set under the assumption that features cannot be accessed entirely at one time. To achieve this, we designed a robust online substitution method to alternately estimate the optimal uncorrupted set and substitute the retained feature set with newly updated features. We demonstrate that our algorithm can recover regression coefficients with a restricted error bound compared to ground truth. Extensive experiments on massive simulation data demonstrated that the proposed algorithm outperforms other competing methods in both effectiveness and efficiency.

\bibliographystyle{IEEEtran}
\bibliography{icdm18_roofs}

\end{document}